\documentclass{article}
\usepackage{ijcai18_dspl}

\renewcommand{\baselinestretch}{0.9}

\usepackage{graphicx, subfigure}
\graphicspath{ {./images/} }

\usepackage{url}

\usepackage{mathtools}
\usepackage{amsmath}
\usepackage{amssymb}
\usepackage{bbm}
\usepackage{bm}
\DeclareMathOperator*{\argminA}{arg\,min}

\DeclareMathOperator*{\mae}{MAE}

\DeclarePairedDelimiter\norm{\lVert}{\rVert}%

\newcommand\nth{\textsuperscript{th}\xspace}

\usepackage{amsthm}
\newtheorem{theorem}{Theorem}
\newtheorem{assumption}{Assumption}

\newtheorem{lemma}[theorem]{Lemma}

\usepackage{graphicx}

\usepackage{tabularx, booktabs}
\newcolumntype{Y}{>{\centering\arraybackslash}X}

\usepackage{times}

\usepackage[linesnumbered,ruled,vlined,algo2e]{algorithm2e}
\title{Multi-view Self-Paced Robust Learning}
\title{Distributed Self-Paced Robust Learning via \\Consensus Alternating Direction Method of Multipliers}
\title{Distributed Self-Paced Learning in Robust Regression}
\title{Distributed Self-Paced Robust Learning}
\title{Distributed Self-Paced Learning}
\title{Distributed Self-Paced Learning in Alternating \\Direction Method of Multipliers}

\author{Xuchao Zhang{$^1$}, Liang Zhao{$^2$}, Zhiqian Chen{$^1$}, Chang-Tien Lu{$^1$} \\ 
	{$^1$}Discovery Analytics Center, Virginia Tech, Falls Church, VA, USA\\
	{$^2$}George Mason University, Fairfax, VA, USA\\
	{$^1$}\{xuczhang, czq, ctlu\}@vt.edu, {$^2$}lzhao9@gmu.edu}

\begin{document}

\maketitle
%\onecolumn
\begin{abstract}
%Self-paced learning (\textit{SPL}) mimics the cognitive process of humans that often learn from easy samples to hard ones. One key issue in \textit{SPL} is the training process of each instance weight is dependent on the other samples which makes it cannot easily be run in a distributed manner in a scalable dataset. In this paper, we propose a novel distributed self-paced learning method (\textit{DSPL}) makes it can handle large scale data in parallel. Specifically, the consensus alternating direction method of multipliers is utilized to solve the problem. We also prove the convergence of our algorithm under some mild condition. Extensive experiments on both synthetic and real datasets demonstrate that our approach is superior to those of existing methods.

Self-paced learning (\textit{SPL}) mimics the cognitive process of humans, who generally learn from easy samples to hard ones. One key issue in \textit{SPL} is the training process required for each instance weight depends on the other samples and thus cannot easily be run in a distributed manner in a large-scale dataset. In this paper, we reformulate the self-paced learning problem into a distributed setting and propose a novel Distributed Self-Paced Learning method (\textit{DSPL}) to handle large scale datasets. Specifically, both the model and instance weights can be optimized in parallel for each batch based on a consensus alternating direction method of multipliers. We also prove the convergence of our algorithm under mild conditions. Extensive experiments on both synthetic and real datasets demonstrate that our approach is superior to those of existing methods.

\end{abstract}

\section{Introduction}

Inspired by the learning processes used by humans and animals \cite{bengio2009curriculum}, Self-Paced Learning (\textit{SPL}) \cite{kumar2010self} considers training data in a meaningful order, from easy to hard, to facilitate the learning process. Unlike standard curriculum learning \cite{bengio2009curriculum}, which learns the data in a predefined curriculum design based on prior knowledge, \textit{SPL} learns the training data in an order that is dynamically determined by feedback from the individual learner, which means it can be more extensively utilized in practice. In self-paced learning, given a set of training samples along with their labels, a parameter $\lambda$ is used to represents the ``age" of the \textit{SPL} in order to control the learning pace. When $\lambda$ is small, ``easy" samples with small losses are considered. As $\lambda$ grows, ``harder" samples with larger losses are gradually added to the training set. This type of learning process is modeled on the way human education and cognition functions. For instance, students will start by learning easier concepts (e.g., Linear Equations) before moving on to more complex ones (e.g., Differential Equations) in the mathematics curriculum. Self-paced learning can also be finely explained in a robust learning manner, where uncorrupted data samples are likely to be used for training earlier in the process than corrupted data. 
%In the process of self-paced learning, the value of $\lambda$ is gradually increased to learn new samples.

In recent years, self-paced learning \cite{kumar2010self} has received widespread attention for various applications in machine learning, such as image classification \cite{AAAI159750}, event detection \cite{jiang2014easy,Zhang:2017:SEF:3132847.3132996} and object tracking \cite{supancic2013self,Zhang:2016:BSD:3061053.3061115}. A wide assortment of \textit{SPL}-based methods \cite{Pi:2016:SBL:3060832.3060891,ma2017self} have been developed, including self-paced curriculum learning \cite{AAAI159750}, self-paced learning with diversity \cite{jiang2014self}, multi-view \cite{xu2015multi} and multi-task \cite{li2017self,ijcai2017-351} self-paced learning. In addition, several researchers have conducted theoretical analyses of self-paced learning. \cite{meng2015objective} provides a theoretical analysis of the robustness of \textit{SPL}, revealing that the implicit objective function of \textit{SPL} has a similar configuration to a non-convex regularized penalty. Such regularization restricts the contributions of noisy examples to the objective, and thus enhances the learning robustness. \cite{ma2017convergence} proved that the learning process of \textit{SPL} always converges to critical points of its implicit objective under mild conditions, while \cite{Fan2017SelfPacedLA} studied a group of new regularizers, named self-paced implicit regularizers that are derived from convex conjugacy.

Existing self-paced learning approaches typically focus on modeling the entire dataset at once; however, this may introduce a bottleneck in terms of memory and computation, as today's fast-growing datasets are becoming too large to be handled integrally. For those seeking to address this issue, the main challenges can be summarized as follows: 
1) \textit{Computational infeasibility of handling the entire dataset at once.} Traditional self-paced learning approaches gradually grow the training set according to their learning pace. However, this strategy fails when the training set grows too large to be handled due to the limited capacity of the physical machines. Therefore, a scalable algorithm is required to extend the existing self-paced learning algorithm for massive datasets.
2) \textit{Existence of heterogeneously distributed ``easy" data.} Due to the unpredictability of data distributions, the ``easy" data samples can be arbitrarily distributed across the whole dataset. Considering the entire dataset as a combination of multiple batches, some batches may contain large amount of ``hard" samples. Thus, simply applying self-paced learning to each batch and averaging across the trained models is not an ideal approach, as some models will only focus on the ``hard" samples and thus degrade the overall performance.
3) \textit{Dependency decoupling across different data batches.} In self-paced learning, the instance weights depend on the model trained on the entire dataset. Also, the trained model depends on the weights assigned to each data instance. If we consider each data batch independently, a model trained in a ``hard" batch can mistakenly mark some ``hard" samples as ``easy" ones. For example, in robust learning, the corrupted data samples are typically considered as ``hard" samples, then more corrupted samples will therefore tend to be involved into the training process when we train data batches independently.

%A naive solution is to separately train the model for each data batch, then average all the model weights together. But the solution contains two main defects: 1) 

In order to simultaneously address all these technical challenges, this paper presents a novel Distributed Self-Paced Learning (\textit{DSPL}) algorithm. The main contributions of this paper can be summarized as follows: 
1) We reformulate the self-paced problem into a distributed setting. Specifically, an auxiliary variable is introduced to decouple the dependency of the model parameters for each data batch. 
2) A distributed self-paced learning algorithm based on consensus ADMM is proposed to solve the \textit{SPL} problem in a distributed setting. The algorithm optimizes the model parameters for each batch in parallel and consolidates their values in each iteration. 
3) A theoretical analysis is provided for the convergence of our proposed \textit{DSPL} algorithm. The proof shows that our new algorithm will converge under mild assumptions, e.g., the loss function can be non-convex. 
4) Extensive experiments have been conducted utilizing both synthetic and real-world data based on a robust regression task. The results demonstrate that the proposed approaches consistently outperform existing methods for multiple data settings. 
To the best of our knowledge, this is the first work to extend self-paced learning to a distributed setting, making it possible to handle large-scale datasets.
%The proposed method based on the consensus ADMM can improve the overall effectiveness of training models compared to utilizing traditional SPL algorithm separately for each batch.

The reminder of this paper is organized as follows. Section \ref{section:problem} gives a formal problem formulation. The proposed distributed self-paced learning algorithm is presented in Section \ref{section:method} and Section \ref{section:analysis} presents a theoretical analysis of the convergence of the proposed method. In Section \ref{section:experiment}, the experimental results are analyzed and the paper concludes with a summary of our work in Section \ref{section:conclusion}.

\section{Problem Formulation}\label{section:problem}
In the context of distributed self-paced learning, we consider the samples to be provided in a sequence of mini batches as $\{(X^{(1)}, \bm y^{(1)}), \dots , (X^{(m)}, \bm y^{(m)})\}$, where $X^{(i)} \in \mathbbm{R}^{p \times n_i}$ represents the sample data for the $i\nth$ batch, $\bm y^{(i)}$ is the corresponding response vector, and $n_i$ is the instance number of the $i\nth$ batch. %

The goal of self-paced learning problem is to infer the model parameter $\bm w \in \mathbbm{R}^{p}$ for the entire dataset, which can be formally defined as follows:
\begin{equation} \label{eq:org_problem}
	\begin{gathered}
		\argminA_{\bm w, \bm v} \sum_{i=1}^m f_i(\bm w, \bm v_i) + \norm{\bm w}_2^2\\
		s.t. \ \ v_{ij} \in [0, 1],\ \ \forall i = 1, \dots, m, \forall j = 1, \dots, n_i
	\end{gathered}
\end{equation}
where $\norm{\bm w}_2^2$ is the regularization term for model parameters $\bm w$. Variable $\bm v_i$ represents the instance weight vector for the $i\nth$ batch and $v_{ij}$ is the weight of the $j\nth$ instance in the $i\nth$ batch. The objective function $f_i(\bm w, \bm v_i)$ for each mini-batch is defined as follows:
\begin{equation} \label{eq:org_obj_function}
	\begin{gathered}
		f_i(\bm w, \bm v_i) = \sum_{j=1}^{n_i} v_{ij} \mathcal{L}(y_{ij}, g(\bm w, \bm x_{ij})) - \lambda \sum_{j=1}^{n_i} v_{ij}
	\end{gathered}
\end{equation}
We denote $\bm x_{ij} \in \mathbbm{R}^{p}$ and $y_{ij} \in \mathbbm{R}$ as the feature vector and its corresponding label for the $j\nth$ instance in the $i\nth$ batch. The loss function $\mathcal{L}$ is used to measure the error between label $y_{ij}$ and the estimated value from model $g$. The term $- \lambda \sum_{j=1}^{n_i} v_{ij}$ is the regularization term for instance weights $\bm v_i$, where parameter $\lambda$ controls the learning pace. The notations used in this paper are summarized in Table \ref{table:math_notation}. 

The problem defined above is very challenging in the following three aspects. 
First, data instances for all $m$ batches can be too large to be handled simultaneously in their entirety, thus requiring the use of a scalable algorithm for large datasets.
Second, the instance weight variable $\bm v_i$ of each batch depends on the optimization result for $\bm w$ shared by all the data, which means all the batches are inter-dependent and it is thus not feasible to run them in parallel.
Third, the objective function of variables $\bm w_i$ and $\bm v_i$ are jointly non-convex and it is an NP-hard problem to retrieve the global optimal solution \cite{gorski2007biconvex}. In next section, we present a distributed self-paced learning algorithm based on consensus ADMM to address all these challenges.

\begin{table}[tb]
	\caption{Mathematical Notations}
	\centering
	\scriptsize
	\label{table:math_notation}
	\tabcolsep=0.1cm
	\scalebox{1.3}{
		\begin{tabular}{ l l }
			%\hline
			\toprule
			%\multicolumn{3}{ |c| }{Team sheet} \\
			\textbf{Notations} & \textbf{Explanations} \\ 
			\hline			
			$p$  & feature number in data matrix $X^{(i)}$ \\
			$n_i$  &  instance number in the $i\nth$ data batch \\ 
			$X^{(i)}$ & data matrix of the $i\nth$ batch \\
			$\bm y^{(i)}$  & the response vector of the $i\nth$ batch \\
			$\bm w$  & model parameter of the entire dataset \\
			$\bm v_i$  & instance weight vector of the $i\nth$ batch \\
			$v_{ij}$  & weight of the $j\nth$ instance in the $i\nth$ batch \\
			$\lambda$  & parameter to control the learning pace \\
			$\mathcal{L}$  & loss function of estimated model \\
			\bottomrule

		\end{tabular}
	}
\end{table}

\section{The Proposed Methodology}\label{section:method}

In this section, we propose a distributed self-paced learning algorithm based on the alternating direction method of multipliers (ADMM) to solve the problem defined in Section \ref{section:problem}.

The problem defined in Equation \eqref{eq:org_problem} cannot be solved in parallel because the model parameter $\bm w$ is shared in each batch and the result of $\bm w$ will impact on the instance weight variable $\bm v_i$ for each batch. In order to decouple the relationships among all the batches, we introduce different model parameters $\bm w_i$ for each batch and use an auxiliary variable $\bm z$ to ensure the uniformity of all the model parameters. The problem can now be reformulated as follows:
\begin{equation}\label{eq:relax_problem}
	\begin{gathered}
		\argminA_{\bm w_i, \bm v_i, \bm z} \sum_{i=1}^m f_i(\bm w_i, \bm v_i; \lambda) + \norm{\bm z}_2^2\\
		%\argminA_{\bm w_i, \bm v_i, \bm z} \sum_{i=1}^m f_i(\bm w_i, \bm v_i; \lambda)\\
		s.t. \ \ v_{ij} \in [0, 1],\ \ \forall i = 1, \dots, m, \forall j = 1, \dots, n_i \\
		\bm w_i - \bm z = 0,\ \ \forall i = 1, \dots , m		
	\end{gathered}
\end{equation}
where the function $f_i(\bm w_i, \bm v_i)$ is defined as follows.
\begin{equation} \label{eq:obj_function}
	\begin{gathered}
		f_i(\bm w_i, \bm v_i; \lambda) = \sum_{j=1}^{n_i} v_{ij} \mathcal{L}(y_{ij}, g(\bm w_i, \bm x_{ij})) - \lambda \sum_{j=1}^{n_i} v_{ij}
		%f_i(\bm w_i, \bm v_i; \lambda) = \sum_{j=1}^{n_i} v_{ij} \mathcal{L}(y_{ij}, g(\bm w_i, \bm x_{ij})) + \theta \norm{\bm w_i}_2^2 - \lambda \sum_{j=1}^{n_i} v_{ij}
	\end{gathered}
\end{equation}

%\begin{equation}\label{eq:relax_problem}
%	\begin{gathered}
%		\argminA_{\bm w_i, \bm v_i, \bm z} \sum_{i=1}^m f_i(\bm w_i, \bm v_i)\\
%		s.t. \ \ v_{ij} \in [0, 1],\ \ \forall i = 1, \dots, m, \forall j = 1, \dots, n_i \\
%		\bm w_i - \bm z = 0,\ \ \forall i = 1, \dots , m		
%	\end{gathered}
%\end{equation}
%where the function $f_i(\bm w_i, \bm v_i)$ is defined as follows.
%\begin{equation} \label{eq:obj_function}
%	\begin{gathered}
%		f_i(\bm w_i, \bm v_i) = \sum_{j=1}^{n_i} v_{ij} \mathcal{L}(y_{ij}, g(\bm w_i, \bm x_{ij}))  + \norm{\bm w_i}_2^2 - \lambda \sum_{j=1}^{n_i} v_{ij} 
%	\end{gathered}
%\end{equation}
Unlike the original problem defined in Equation \eqref{eq:org_problem}, here each batch has its own model parameter $\bm w_i$ and the constraint $\bm w_i - \bm z = 0$ for $\forall i=1, \dots, m$ ensures the model parameter $\bm w_i$ has the same value as the auxiliary variable $\bm z$. The purpose of the problem reformulation is to optimize the model parameters $\bm w_i$ in parallel for each batch. It is important to note that the reformulation is \textit{tight} because our new problem is equivalent to the original problem when the constraint is strictly satisfied.

In the new problem, Equation \eqref{eq:relax_problem} is a bi-convex optimization problem over $\bm v_i$ and $\bm w_i$ for each batch with fixed $\bm z$, which can be efficiently solved using the Alternate Convex Search (ACS) method \cite{gorski2007biconvex}. With the variable  $\bm v$ fixed, the remaining variables $\{\bm w_i\}$, $\bm z$ and $\bm \alpha$ can be solved by consensus ADMM \cite{boyd2011distributed}. As the problem is an NP-hard problem, in which the global optimum requires polynomial time complexity, we propose an alternating algorithm \textit{DSPL} based on ADMM to handle the problem efficiently. 

The augmented Lagrangian format of optimization in Equation \eqref{eq:relax_problem} can be represented as follows:
\begin{equation} \label{eq:lag}
	\begin{aligned}
		L &= \sum_{i=1}^m f_i(\bm w_i, \bm v_i; \lambda) + \norm{\bm z}_2^2 + \sum_{i=1}^m \bm \alpha_i^T(\bm w_i - \bm z) \\
		&\ \ \ \ \ \ \ \ \ \ \ + \frac{\rho}{2} \sum_{i=1}^m \norm{\bm w_i - \bm z}_2^2
		%L &= \sum_{i=1}^m f_i(\bm w_i, \bm v_i; \lambda) + \sum_{i=1}^m \bm \alpha_i^T(\bm w_i - \bm z) + \frac{\rho}{2} \sum_{i=1}^m \norm{\bm w_i - \bm z}_2^2\\		
	\end{aligned}
\end{equation}

where $\{\alpha_i\}_{i=1}^m$ are the Lagrangian multipliers and $\rho$ is the step size of the dual step. 

%With this setting, the Equation (3) is  Therefore, we will 
%Given fixed instance weights $\bm v$, the variable $\bm w$, $\bm z$ and $\bm \alpha$ can be solved by consensus ADMM. 
The process used to update model parameter $\bm w_i$ for the $i\nth$ batch with the other variables fixed is as follows: %and instance weight vector $\bm v_i$ for each batch is shown as follows:
\begin{equation} \label{eq:problem_w}
	\begin{aligned}
		\bm w_i^{k+1} &= \argminA_{\bm w_i} f_i(\bm w_i, \bm v_i; \lambda) + [\bm \alpha_i^k]^T(\bm w_i - \bm z^k) \\ 
		&\ \ \ \ \ \ \ \ \ \ \ \ \ \ \ \ \ \ + \frac{\rho}{2} \norm{\bm w_i - \bm z^k}_2^2 \\
		%\lambda &= \lambda * \mu
	\end{aligned}
\end{equation}

Specifically, if we choose the loss function $\mathcal{L}$ to be a squared loss and model $g(\bm w, \bm x_{ij})$ to be a linear regression $g(\bm w, \bm x_{ij}) = \bm w^T \bm x_{ij}$, we have the following analytical solution for $\bm w_i$:
%\begin{equation}
%	\begin{aligned}
%		\bm w_i^{k+1} =& \bigg(2\sum_{j=1}^{n_i} \bm v_{ij} \bm x_{ij} \bm x_{ij}^T + (\rho+2)I \bigg)^{-1}\\
%		&\ \ \ \ \ \ \ \ \ \ \ \ \cdot \bigg(2\sum_{j=1}^{n_i} \bm v_{ij} \bm x_{ij} y_{ij} - \bm \alpha_i^k + \rho \bm z^k \bigg) \\
%		\bm v^{k+1}_i =& \mathbbm{1}\bigg(\mathcal{L} \big(y_{ij}, g(\bm w_i^{k+1}, \bm x_{ij})\big) < \lambda \bigg)
%	\end{aligned}
%\end{equation}
\begin{equation}
	\begin{aligned}
		\bm w_i^{k+1} =& \bigg(2\sum_{j=1}^{n_i} \bm v_{ij} \bm x_{ij} \bm x_{ij}^T + \rho \cdot I \bigg)^{-1}\\
		&\ \ \ \ \ \ \ \ \ \cdot \bigg(2\sum_{j=1}^{n_i} \bm v_{ij} \bm x_{ij} y_{ij} - \bm \alpha_i^k + \rho \bm z^k \bigg) \\
	\end{aligned}
\end{equation}

The auxiliary variable $\bm z$ and Lagrangian multipliers $\bm \alpha_i$ can be updated as follows:
\begin{equation} \label{eq:z_alpha_update}
	\begin{gathered}
		\bm z^{k+1} = \frac{\rho}{2+\rho m} \sum_{i=1}^m (\bm w_i^{k+1} + \frac{1}{\rho}\bm \alpha_i^k) \\
		\bm \alpha_i^{k+1} = \bm \alpha_i^k + \rho(\bm w_i^{k+1} - \bm z^{k+1}) \\
	\end{gathered}
\end{equation}

The stop condition of consensus ADMM is determined by the (squared) norm of the primal and dual residuals of the $k\nth$ iteration, which are calculated as follows:
\begin{equation}
	\begin{gathered}
		\norm{r_k}_2^2 = \sum_{i=1}^m \norm{\bm w_i^k - \bm z^k}_2^2 \\
		\norm{s_k}_2^2 = m \rho^2 \norm{\bm z^{k} - \bm z^{k-1}}_2^2 \\
	\end{gathered}
\end{equation}

After the weight parameter $\bm w_i$ for each batch has been updated, the instance weight vector $\bm v_i$ for each batch will be updated based on the fixed $\bm w_i$ by solving the following problem: 
\begin{equation} \label{eq:problem_v}
	\begin{aligned}
		\bm v_i^{t+1} = \argminA_{\bm v_i} \sum_{j=1}^{n_i} v_{ij} \mathcal{L}(y_{ij}, g(\bm w_i^{t+1}, \bm x_{ij})) - \lambda \sum_{j=1}^{n_i} v_{ij}
	\end{aligned}
\end{equation}

For the above problem in Equation \eqref{eq:problem_v}, we always obtain the following closed-form solution: 
\begin{equation}
	\begin{aligned}
		\bm v^{t+1}_i =& \mathbbm{1}\bigg(\mathcal{L} \big(y_{ij}, g(\bm w_i^{t+1}, \bm x_{ij})\big) < \lambda \bigg)
	\end{aligned}
\end{equation}
where $\mathbbm{1}(\cdot)$ is the indicator function whose value equals to one when the condition $\mathcal{L} \big(y_{ij}, g(\bm w_i^{t+1}, \bm x_{ij})\big) < \lambda$ is satisfied; otherwise, its value is zero.

%Classification with SVM \url{https://github.com/DylanMuir/fmin_adam}
The details of algorithm \textit{DSPL} are presented in Algorithm \ref{algo:dspl}. In Lines 1-2, the variables and parameters are initialized. With the variables $\bm v_i$ fixed, the other variables are optimized in Lines 5-13 based on the result of consensus ADMM, in which both the model weights $\bm w_i$ and Lagrangian multipliers $\bm \alpha_i$ can be updated in parallel for each batch. Note that if no closed-form can be found for Equation \eqref{eq:problem_w}, the updating of $\bm w_i$ is the most time-consuming operation in the algorithm. Therefore, updating $\bm w_i$ in parallel can significantly improve the efficiency of the algorithm. The variable $\bm v_i$ for each batch is updated in Line 14, with the variable $\bm w_i$ fixed. In Lines 15-18, the parameter $\lambda$ is enlarged to include more data instances into the training set. $\tau_\lambda$ is the maximum threshold for $\lambda$ and $\mu$ is the step size. The algorithm will be stopped when the Lagrangian is converged in Line 20. The following two alternative methods can be applied to improve the efficiency of the algorithm: 1) dynamically update the penalty parameter $\rho$ after Line 11. When $r > 10s$, we can update $\rho \gets 2\rho$. When $10r < s$, we have $\rho \gets \rho / 2$. 2) Move the update of variable $\bm v_i$ into the consensus ADMM step after Line 9. This ensures that the instance weights are updated every time the model is updated, so that the algorithm quickly converges. However, no theoretical convergence guarantee can be made for the two solutions, although in practice they do always converge.

\begin{algorithm2e}[t]
	\small
	\DontPrintSemicolon % Some LaTeX compilers require you to use \dontprintsemicolon instead
	\KwIn{$X \in \mathbbm{R}^{p \times n}$, $\bm y \in \mathbbm{R}^{n}$, $\lambda_0 \in \mathbbm{R}$, $\tau_\lambda \in \mathbbm{R}$, $\mu \in \mathbbm{R}$}
	\KwOut{solution $\bm w^{(t+1)}$, $\bm v^{(t+1)}$}
	Initialize $\bm w_i^{0} = \bm 1$, $\bm v_i^{0} = \bm 1$ \\
	Choose $\varepsilon_L > 0$, $\varepsilon_r > 0$, $\varepsilon_s > 0$, $\lambda \leftarrow \lambda_0$, $t \leftarrow 0$ \\
	%$\tilde{\bm w} \leftarrow \argminA_{\bm w} \sum_{i=1}^k \mathcal{L} \big(y_i, f(\bm x_i, \bm w)\big)$ \\
	\Repeat{$\norm{L^{t+1} - L^{t}}_2 < \varepsilon_L$}
	{
		$k \leftarrow 0$ \\
		\Repeat{$\norm{r^{k+1}}_2^2 < \varepsilon_r$ and $\norm{s^{k+1}}_2^2 < \varepsilon_s$}
		{			
			$\bm z^{k+1} \leftarrow \frac{1}{m} \sum_{i=1}^m (\bm w_i^{k+1} + \frac{1}{\rho}\bm \alpha_i^k)$ \\
			%			$\bm w_i^{k+1}, \bm v_i^{k+1} \leftarrow \argminA_{\bm w_i, \bm v_i} f_i(\bm w_i, \bm v_i) + [\bm \alpha_i^k]^T(\bm w_i - \bm z^k) + \frac{\rho}{2} \norm{\bm w_i - \bm z^k}_2^2$ \\
			%Parallelly Update $\bm w_i^{k+1}, \bm v_i^{k+1} \leftarrow \argminA_{\bm w_i, \bm v_i} f_i(\bm w_i, \bm v_i) + $ \\
			%Parallelly update variables $\bm w_i^{k+1}, \bm v_i^{k+1}$ \\
			Update variables $\bm w_i^{k+1}$ in parallel, for $\forall i = 1 \dots m$ \\
			\ \ \ \ \ \ \ \ \ $\bm w_i^{k+1} \leftarrow \argminA f_i(\bm w_i, \bm v_i) + $ \\ 
			\ \ \ \ \ \ \ \ \ \ \ \ \ \ \ \ \ \ \ \ \ \ \ \ \ \ \ \ \ \ \ \ \ \ \ \ \ \ \ $ [\bm \alpha_i^k]^T(\bm w_i - \bm z^k) + \frac{\rho}{2} \norm{\bm w_i - \bm z^k}_2^2$ \\
			%\ \ \ \ \ \ \ \ \ $\bm v^{k+1}_i \leftarrow \mathbbm{1}\bigg(\mathcal{L} \big(y_{ij}, g(\bm w_i^{k+1}, \bm x_{ij})\big) < \lambda \bigg)$\\
			Update dual $\bm \alpha_i^{k+1}$ $\leftarrow \bm \alpha_i^k + \rho(\bm w_i^{k+1} - \bm z^{k+1})$ in parallel \\
			Update primal and dual residuals $r^{k+1}$ and $s^{k+1}$. \\
%			\uIf{$r>10s$}{$\rho \gets 2\rho$}
%			\uElseIf{$10r < s$}{$\rho \gets \rho / 2$}
			$k \leftarrow k + 1$
		}
		$\bm v^{t+1}_i \leftarrow \mathbbm{1}\bigg(\mathcal{L} \big(y_{ij}, g(\bm w_i^{t+1}, \bm x_{ij})\big) < \lambda \bigg)$, for $\forall i = 1 \dots m$\\
		%\uIf{$\exists i, v_i \ne 1$}{$\lambda \leftarrow \lambda * \mu$} 
		\uIf{$\lambda < \tau_\lambda$}{$\lambda \leftarrow \lambda * \mu$}
		\uElse{$\lambda \leftarrow \tau_\lambda$}
		$t \leftarrow t + 1$ \\
	}
	\textbf{return} $\bm z^{t+1}$, $\bm v^{t+1}$
	\caption{{\sc Dspl Algorithm}}
	\label{algo:dspl}
\end{algorithm2e}

\section{Theoretical Analysis}\label{section:analysis}
In this section, we will prove the convergence of the proposed algorithm.
Before we start to prove the convergence of Algorithm \ref{algo:dspl}, we make the following assumptions regarding our objective function and penalty parameter $\rho$:

\begin{assumption} [Gradient Lipchitz Continuity] There exists a positive constant $\varphi_i$ for objective function $f_i(\bm w_i)$ of each batch with the following properties:
	\begin{equation}
		\begin{gathered}
			\norm{\triangledown_{\bm w_i} f_i(\bm x_i) - \triangledown_{\bm w_i} f_i(\bm y_i)} \le \varphi_i \norm{\bm x_i - \bm y_i}, \\
			\forall \bm x_i, \bm y_i, i = 1, \dots, m
		\end{gathered}
	\end{equation}

\end{assumption}

\begin{assumption} [Lower Bound] The objective function in problem \eqref{eq:relax_problem} has a lower bound $\mathcal{B}$ as follows:
\begin{equation}
	\begin{gathered}
		\mathcal{B} = \min_{\bm w_i, \bm v_i, \bm z} \bigg \{ \sum_{i=1}^m f_i(\bm w_i, \bm v_i) + \norm{\bm z}_2^2 \bigg \} > - \infty
	\end{gathered}
\end{equation}

\end{assumption}

\begin{assumption} [Penalty Parameter Constraints]
For $\forall i = 1 \dots m$, the penalty parameter $\rho_i$ for each batch is chosen in accord with the following constraints:

\begin{itemize}
	\item For $\forall i$, the subproblem \eqref{eq:problem_w} of variable $\bm w_i$ is strongly convex with modulus $\gamma_i(\rho_i)$. 
	\item For $\forall i$, we have $\rho_i \gamma_i(\rho_i) > 2\varphi_i^2$ and $\rho_i \ge \varphi_i$.
\end{itemize}

\end{assumption}
Note that when $\rho_i$ increases, subproblem \eqref{eq:problem_w} will be eventually become strongly convex with respect to variable $\bm w_i$. For simplicity, we will choose the same penalty parameter $\rho$ for all the batches with $\rho = \max_i (\rho_i)$. Based on these assumptions, we can draw the following conclusions.

\begin{lemma}

Assume the augmented Lagrangian format of optimization problem \eqref{eq:relax_problem} satisfies Assumption 1, the augmented Lagrangian $L$ has the following property:
\begin{equation} \label{eq:lemma1}
	\begin{aligned}
		L(\{\bm w_i^{k+1}\}, \bm z^{k+1}, \bm \alpha^{k+1}) \le L(\{\bm w_i^{k}\}, \bm z^{k}, \bm \alpha^{k})
	\end{aligned}
\end{equation}

\end{lemma}

\begin{proof}

Since the the objective function $f_i(\bm w_i)$ for each batch is gradient Lipchitz continuous with a positive constant $\varphi_i$, the Lagrangian in Equation \eqref{eq:lag} has the following property according to Lemma 2.2 in \cite{hong2016convergence}:
\begin{equation} \label{eq:lemma1}
	\begin{aligned}
		&L(\{\bm w_i^{k+1}\}, \bm z^{k+1}, \bm \alpha^{k+1}) - L(\{\bm w_i^{k}\}, \bm z^{k}, \bm \alpha^{k}) \\
		&\le \sum_{i=1}^m \bigg( \frac{\varphi_i^2}{\rho} - \frac{\gamma_i(\rho)}{2} \bigg) \norm{\bm w_i^{k+1} - \bm w_i^k}_2^2 - \frac{\gamma}{2}\norm{\bm z^{k+1} - \bm z^{k}}_2^2 \\
		&\stackrel{(a)}{\le} - \frac{\gamma}{2}\norm{\bm z^{k+1} - \bm z^{k}}_2^2 \le 0
	\end{aligned}
\end{equation}
where $\gamma = m\rho > 0$. The inequality (a) follows from Assumption 2, namely that $\rho \gamma_i(\rho) > 2 \varphi_i^2$, so we have $\bigg( \frac{\varphi_i^2}{\rho} - \frac{\gamma_i(\rho)}{2} \bigg) < 0$.
\end{proof}

\begin{lemma}
Assume the augmented Lagrangian of problem \eqref{eq:relax_problem} satisfies Assumptions 1-3, the augmented Lagrangian $L$ is lower bounded as follows:
\begin{equation} \label{eq:lemma2}
	\begin{aligned}
		\lim_{k \rightarrow \infty}	L(\{\bm w_i^{k+1}\}, \bm z^{k+1}, \bm \alpha^{k+1}) \ge \mathcal{B}
	\end{aligned}
\end{equation}
where $\mathcal{B}$ is the lower bound of the objective function in problem \eqref{eq:relax_problem}.
\end{lemma}
\begin{proof}
Due to the gradient Lipchitz continuity assumption, we have the following optimality condition for the $\bm w_i$ update step in Equation \eqref{eq:problem_w}:
\begin{equation}
	\begin{aligned}
		\triangledown_{\bm w_i} f_i(\bm w_i^{k+1}) + \bm \alpha_i^k + \rho(\bm w_i^{k+1} - \bm z^{k+1}) = 0, \ \forall i = 1 \dots m
	\end{aligned}
\end{equation}
Combined with the update of the Lagrangian multipliers $\bm \alpha_i$ in Equation \eqref{eq:z_alpha_update}, we have 
\begin{equation}
	\begin{aligned}
		\triangledown_{\bm w_i} f_i(\bm w_i^{k+1}) = - \bm \alpha_i^{k+1}, \ \forall i = 1 \dots m
	\end{aligned}
\end{equation}
The augmented Lagrangian can be represented as:
\begin{equation} \label{eq:lemma2_fact1}
	\begin{aligned}
		&L(\{\bm w_i^{k+1}\}, \bm z^{k+1}, \bm \alpha^{k+1}) \\
		%&=\norm{\bm z^{k+1}}_2^2 + \sum_{i=1}^m \bigg( f_i(\bm w_i^{k+1}) + [\bm \alpha_i^{k+1}]^T(\bm w_i^{k+1} - \bm z^{k+1}) \\
		%&\ \ \ \ \ \ \ \ \ \ \ \ \ \ \ \ \ \ \ \ \ \ \ \ \ \ \ \ \ \ \ \ \ + \frac{\rho}{2}\norm{\bm w_i^{k+1} - \bm z^{k+1}}_2^2 \bigg) \\	
		&\stackrel{(a)}{=}\norm{\bm z^{k+1}}_2^2 + \sum_{i=1}^m \bigg( f_i(\bm w_i^{k+1}) + \triangledown_{\bm w_i} f_i(\bm w_i^{k+1}) \cdot  \\
		&\ \ \ \ \ \ \ \ \ \ \ \ \ \ \ \ \ \ \ \ \ \ \ \ \ (\bm z^{k+1} - \bm w_i^{k+1})+ \frac{\rho}{2}\norm{\bm w_i^{k+1} - \bm z^{k+1}}_2^2 \bigg) \\	
		&\stackrel{(b)}{\ge} \norm{\bm z^{k+1}}_2^2 + \sum_{i=1}^m f_i(\bm z^{k+1}) \stackrel{(c)}{\ge} \mathcal{B} 
	\end{aligned}
\end{equation}
Equation (a) follows from Equation \eqref{eq:lemma2_fact1} and the inequality (b) comes from the Lipschitz continuity property in Assumption 1. The inequality (c) follows from the lower bound property in Assumption 2.
	
\end{proof}

\begin{theorem}
	%Supposed that the Assumption 1-3 are satisfied, the augmented Lagrangian $L$ is monotonically decreasing and is convergent, then the Algorithm \ref{algo:dspl} converges.
	The Algorithm \ref{algo:dspl} converges when Assumptions 1-3 are all satisfied.
\end{theorem}

\begin{proof}
In Lemmas 1 and 2, we proved that the Lagrangian is monotonically decreasing and has a lower bound during the iterations of ADMM. Now we will prove that the same properties hold for the entire algorithm after updating variable $\bm v$ and parameter $\lambda$.
\begin{align*}
	&L(\{\bm w^{t+1}\}, \bm v^{t+1}, \bm z^{t+1}, \bm \alpha^{t+1}; \lambda^{t+1}) \\
	&\stackrel{(a)}{\le} L(\{\bm w^{t}\}, \bm v^{t+1}, \bm z^{t}, \bm \alpha^{t}; \lambda^{t+1}) \stackrel{(b)}{\le} L(\{\bm w^{t}\}, \bm v^{t}, \bm z^{t}, \bm \alpha^{t}; \lambda^{t+1})\\
	&{=} L(\{\bm w^{t}\}, \bm v^{t}, \bm z^{t}, \bm \alpha^{t}; \lambda^{t}) + (\lambda^{t} - \lambda^{t+1}) \sum_{i=1}^m \sum_{j=1}^{n_i} v_{ij}^t \\
	&\stackrel{(c)}{\le} L(\{\bm w^{t}\}, \bm v^{t}, \bm z^{t}, \bm \alpha^{t}; \lambda^{t})
\end{align*}
Inequality (a) follows Lemma 1 and inequality (b) follows the optimization step in Line 14 in Algorithm \ref{algo:dspl}. Inequality (c) follows from the fact that $\lambda$ increases monotonically so that $\lambda^t \le \lambda^{t+1}$.
%Because the value of , then  Then we have $(\lambda^{t} - \lambda^{t+1}) \sum_{i=1}^m \sum_{j=1}^{n_i} v_{ij}^t \le 0$. Therefore, we have
As $L(\{\bm w^{t+1}\}, \bm z^{t+1}, \bm \alpha^{t+1})$ for some constant values of $\bm v$ and $\lambda$ has a lower bound $\mathcal{B}$, we can easily prove that $L(\{\bm w^{t+1}\}, \bm v^{t+1}, \bm z^{t+1}, \bm \alpha^{t+1}; \lambda^{t+1}) \ge \mathcal{B} + C - \tau_\lambda n$, where $C$ is a constant and $n$ is the size of the entire dataset. Therefore, the Lagrangian $L$ is convergent. According to the stop condition for Algorithm \ref{algo:dspl}, the algorithm converges when the Lagrangian $L$ is converged.

\end{proof}
%\section{Related Work} \label{section:related_work}

\section{Experimental Results}\label{section:experiment}

In this section, the performance of the proposed algorithm \textit{DSPL} is evaluated for both synthetic and real-world datasets in robust regression task. After the experimental setup has been introduced in Section \ref{section:exp_setup}, we present the results for the regression coefficient recovery performance with different settings using synthetic data in Section \ref{section:regression_recovery}, followed by house rental price prediction evaluation using real-world data in Section \ref{section:airbnb}. All the experiments were performed on a 64-bit machine with an Intel(R) Core(TM) quad-core processor (i7CPU@3.6GHz) and 32.0GB memory. Details of both the source code and the datasets used in the experiment can be downloaded here\footnote{{https://goo.gl/cis7tK}}.
\begin{figure*}[ht]
	\centering
	\scalebox{0.99}{
		\subfigure[p=100, n=10K, b=10, dense noise]{%
			\label{fig:beta_1} 
			\includegraphics[trim=0.6cm 0.1cm 0.6cm 0.1cm,width=0.31\linewidth]{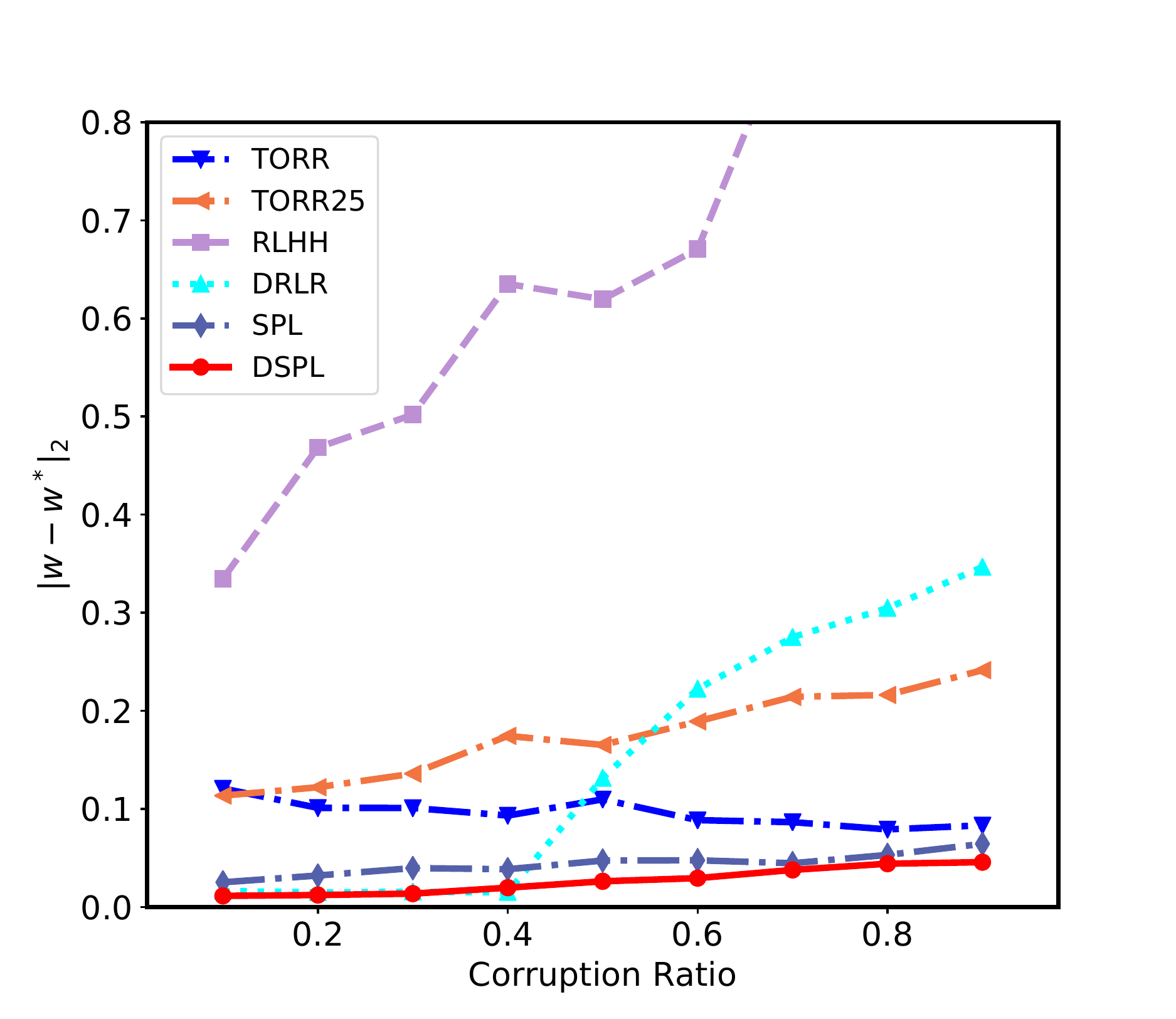}
	}}
	\scalebox{0.99}{
		\subfigure[p=400, n=10K, b=10, dense noise]{%
			\label{fig:beta_2}
			\includegraphics[trim=0.6cm 0.1cm 0.6cm 0.1cm,width=0.31\linewidth]{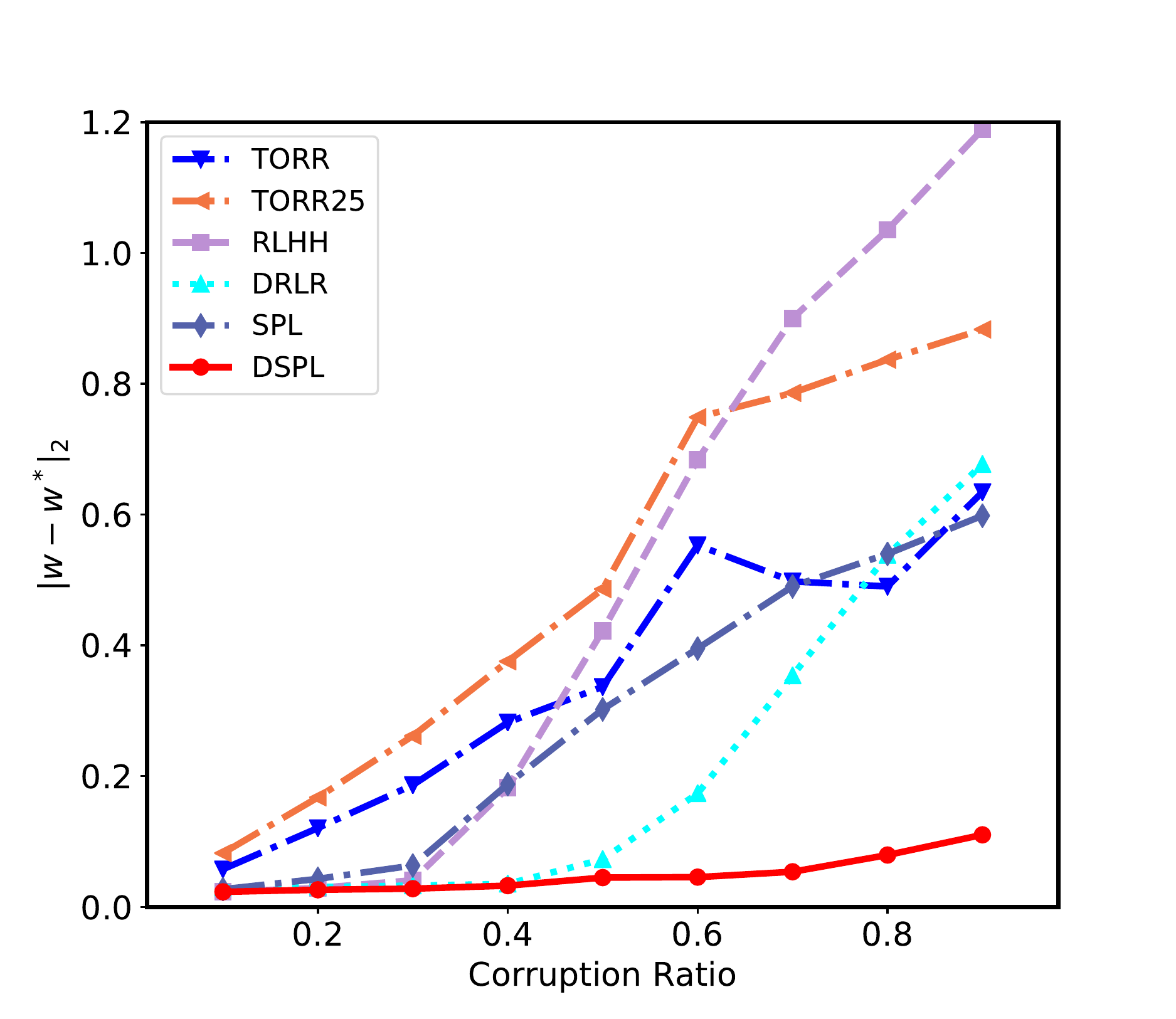}
	}}
	\scalebox{0.99}{
		\subfigure[p=100, n=50K, b=10, dense noise]{%
			\label{fig:beta_3}
			\includegraphics[trim=0.6cm 0.1cm 0.6cm 0.1cm,width=0.31\linewidth]{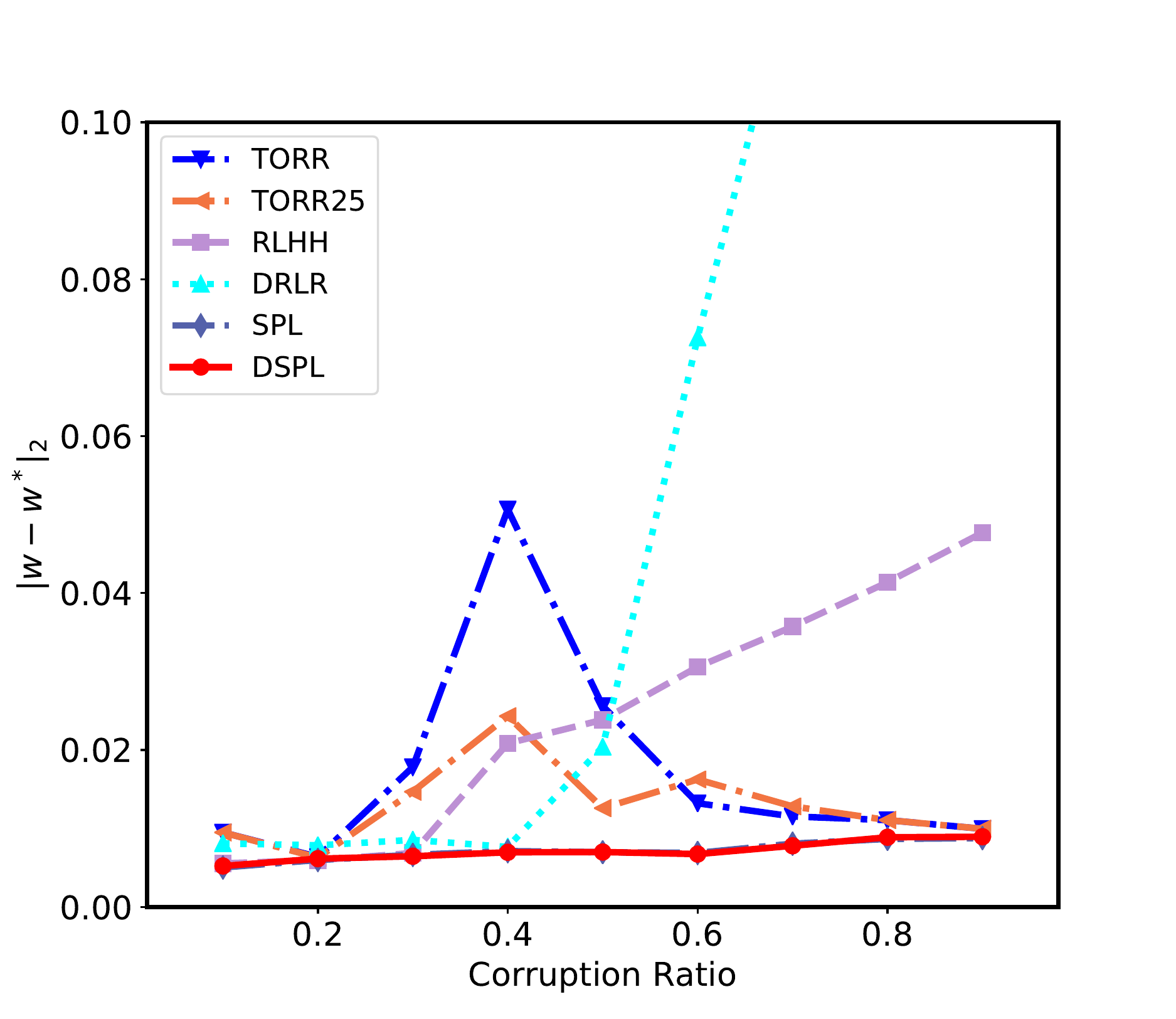}
	}}
%	\scalebox{0.94}{
%		\subfigure[p=100, n=10K, b=20, dense noise]{%
%			\label{fig:beta_4}
%			\includegraphics[trim=0.6cm 0.1cm 0.6cm 0.1cm,width=0.31\linewidth]{beta_4}
%	}}  
%	\scalebox{0.94}{
%		\subfigure[p=100, n=10K, b=10, no dense noise]{%
%			\label{fig:beta_5}
%			\includegraphics[trim=0.6cm 0.1cm 0.6cm 0.1cm,width=0.31\linewidth]{beta_5}
%	}}
%	\scalebox{0.94}{
%		\subfigure[p=200, n=10K, b=10, no dense noise]{%
%			\label{fig:beta_6}
%			\includegraphics[trim=0.6cm 0.1cm 0.6cm 0.1cm,width=0.31\linewidth]{beta_6}
%	}}
	
	\caption{%
		\small Regression coefficient recovery performance for different corruption ratios.
	}%
	\label{fig:beta}
\end{figure*}

\begin{figure}[h]
	\centering
	\scalebox{0.8}{
		\includegraphics[trim=1cm 0.1cm 1cm 0.1cm,width=0.95\linewidth]{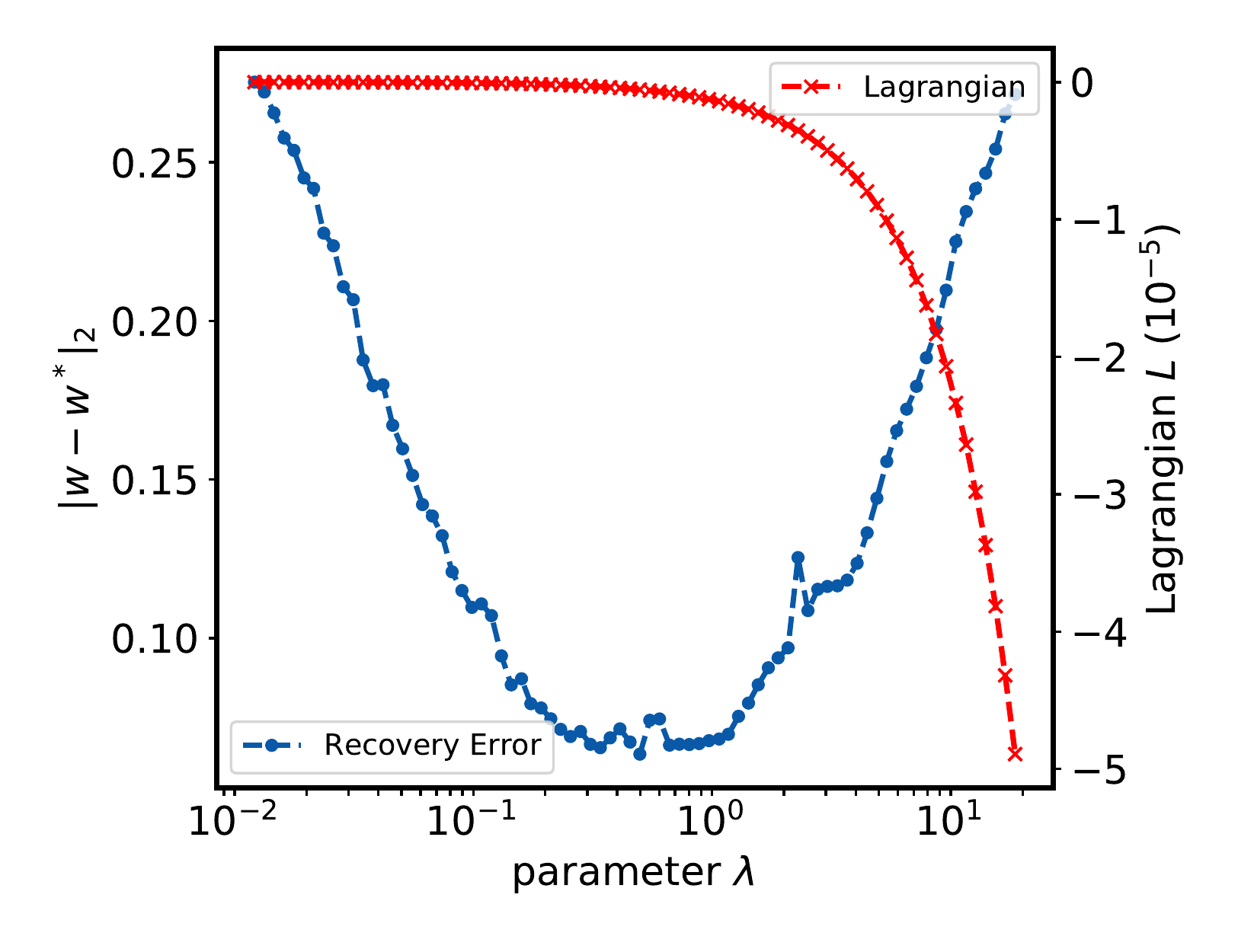}
	}
%	\scalebox{0.94}{
%		\subfigure[Lagrangian Convergence]{%
%			\label{fig:beta_2}
%			\includegraphics[trim=0.6cm 0.1cm 0.6cm 0.1cm,width=0.49\linewidth]{lag}
%	}}

	\caption{%
		\small Relationship between parameter $\lambda$ and coefficient recovery error and the corresponding Lagrangian.
	}%
	\label{fig:coeff_lag}
\end{figure}

\subsection{Experimental Setup} \label{section:exp_setup}
\subsubsection{Datasets and Labels} \label{section:dataset}
The datasets used for the experimental verification were composed of synthetic and real-world data. The simulation samples were randomly generated according to the model $\bm y^{(i)} = [X^{(i)}]^T \bm w_* + \bm u ^{(i)} + \bm \varepsilon^{(i)}$ for each mini-batch, where $\bm w_*$ represents the ground truth coefficients and $\bm u^{(i)}$ the adversarial corruption vector. $\varepsilon^{(i)}$ represents the additive dense noise for the $i\nth$ batch, where $\varepsilon_j^{(i)} \sim \mathcal{N}(0, \sigma^2)$. We sampled the regression coefficients $\bm w_* \in \mathbbm{R}^p$ as a random unit norm vector. The covariance data $X^{(i)}$ for each mini-batch was drawn independently and identically distributed from $\bm x_i \sim \mathcal{N}(0, I_p)$ and the uncorrupted response variables were generated as $\bm y^{(i)}_* = {\big[\bm X^{(i)}\big]}^T \bm w_* + \bm \varepsilon^{(i)}$. The corrupted response vector for each mini-batch was generated as $\bm y^{(i)} = \bm y^{(i)}_* + \bm u^{(i)}$, where the corruption vector $\bm u^{(i)}$ was sampled from the uniform distribution $[-5\|\bm y^{(i)}_*\|_\infty, 5\|\bm y^{(i)}_*\|_\infty]$. The set of uncorrupted points was selected as a uniformly random subset of $[n_i]$ for each batch. %We define $\gamma$ as the corruption ratio of the total $m$ mini-batches; $\gamma^{(i)}$ is randomly chosen in the condition of $\gamma = \sum_i^m \gamma^{(i)}$, where $\gamma$ should be less than $1/2$ to ensure the number of uncorrupted samples is greater than the number of corrupted ones.

The real-world datasets utilized consisted of house rental transaction data from two cities, \textit{New York City} and \textit{Los Angeles} listed on the Airbnb\footnote{https://www.airbnb.com/} website from January 2015 to October 2016. These datasets can be publicly downloaded\footnote{{http://insideairbnb.com/get-the-data.html}}. For the \textit{New York City} dataset, the first 321,530 data samples from January 2015 to December 2015 were used as the training data and the remaining 329,187 samples from January to October 2016 as the test data. For the \textit{Los Angeles} dataset, the first 106,438 samples from May 2015 to May 2016 were used as training data, and the remaining 103,711 samples as test data. In each dataset, there were 21 features after data preprocessing, including the number of beds and bathrooms, location, and average price in the area.
\begin{table}[t]
	\caption{Regression Coefficient Recovery Performance for\ \ \ \ \ \ \ \ \ \ \ \ \ \ \ \ \ \ \ \ \ \ \ \ \ \ \ \ \ Different Corrupted Batches}
	\centering
	\small
	%\scriptsize
	\label{table:batchcorr}
	
	\scalebox{1}{
		\begin{tabularx}{0.48\textwidth}{c *{6}{Y}}
			\toprule
			%\cmidrule(lr){2-6} 
			%& \underline{70\%} & 80\% & 90\% & 95\% & 70\% & 80\% & 90\% & 95\%\\
			& \textbf{4/10} & \textbf{5/10} & \textbf{6/10} & \textbf{7/10} & \textbf{8/10} & \textbf{9/10} \\
			\midrule
			
			%\textbf{OLS}		&0.000 & 0.000 & 0.000 & 0.000 & 0.000 & 0.000\\		
			\textbf{TORR}		& 0.093 & 0.109 & 0.088 & 0.086 & 0.079 & 0.083\\	
			\textbf{TORR25}		& 0.174 & 0.165 & 0.189 & 0.214 & 0.216 & 0.241\\		
			\textbf{RLHH}		& 0.635 & 0.619 & 0.670 & 0.907 & 0.851 & 0.932\\		
			\textbf{DRLR}		& \textbf{0.014} & 0.131 & 0.222 & 0.274 & 0.304 & 0.346\\			
			\textbf{SPL}		& 0.038 & 0.047 & 0.047 & 0.044 & 0.053 & 0.064\\
			%\midrule
			\textbf{DSPL}		& 0.030 & \textbf{0.034} & \textbf{0.039} & \textbf{0.036} & \textbf{0.041} & \textbf{0.045}\\
			\bottomrule
		\end{tabularx}
	}
	
\end{table}
\subsubsection{Evaluation Metrics}
For the synthetic data, we measured the performance of the regression coefficient recovery using the averaged $L_2$ error $e = \norm{\hat{\bm w} - \bm w_*}_2$, 
where $\hat{\bm w}$ represents the recovered coefficients for each method and $\bm w_*$ represents the ground truth regression coefficients. 
%To validate the performance for corrupted set discovery, precision, recall, and F1-score are measured by comparing the discovered corrupted sets with the actual ones. 
%To compare the scalability of each method, the CPU running time for each of the competing methods was also measured.
For the real-world dataset, we used the mean absolute error (MAE) to evaluate the performance for rental price prediction. Defining $\hat{\bm y}$ and $\bm y$ as the predicted price and ground truth price, respectively, the mean absolute error between $\hat{\bm y}$ and $\bm y$ can be presented as $\mae (\hat{\bm y},\bm y)= \frac{1}{n} \sum_{i=1}^{n}\big|\hat{y_i} - y_i\big|$, where $y_i$ represents the label of the $i\nth$ data sample.

\begin{table*}[ht]
	\caption{Mean Absolute Error for Rental Price Prediction}
	\centering
	\small
	%\scriptsize
	\label{table:rental_price}

	\scalebox{1.00}{
		\begin{tabularx}{0.85\textwidth}{c *{5}{Y}||c}
			\toprule
			%\cmidrule(lr){2-6} 
			%& \underline{70\%} & 80\% & 90\% & 95\% & 70\% & 80\% & 90\% & 95\%\\
			& \multicolumn{6}{c}{\textbf{New York City} (\textbf{Corruption Ratio})} \\
			\cmidrule(lr){2-7} 
			& \textbf{10\%} & \textbf{30\%} & \textbf{50\%} & \textbf{70\%} & \textbf{90\%} & \textbf{Avg.} \\
			\midrule
			%			\textbf{OLS}		&0.2461$\big/$35.46\% & 0.2437$\big/$00.00\% & \textbf{0.1921}$\big/$00.00\% & 0.2733$\big/$00.00\% & 0.0910$\big/$00.00\% & 0.2092$\big/$00.00\%\\		
			%			\textbf{TORR}		&0.2075$\big/$00.00\% & 0.2157$\big/$00.00\% & 0.1766$\big/$00.00\% & 0.2405$\big/$00.00\% & 0.0971$\big/$00.00\% & 0.1875$\big/$00.00\%\\	
			%			\textbf{RLHH}		&0.2111$\big/$00.00\% & 0.2332$\big/$00.00\% & 0.1867$\big/$00.00\% & 0.2739$\big/$00.00\% & 0.1064$\big/$00.00\% & 0.2023$\big/$00.00\%\\	
			%			\textbf{RMFP-GC}	&0.2146$\big/$00.00\% & 0.2296$\big/$00.00\% & 0.1887$\big/$00.00\% & 0.2552$\big/$00.00\% & \textbf{0.1098}$\big/$00.00\% & 0.1996$\big/$00.00\%\\		
			%			\textbf{RMFP-MV}	&\textbf{0.2472}$\big/$00.00\% & \textbf{0.2442}$\big/$00.00\% & 0.1919$\big/$00.00\% & \textbf{0.2743}$\big/$00.00\% & 0.0967$\big/$00.00\% & \textbf{0.2109}$\big/$00.00\%\\		
			
			%\textbf{OLS}		&0.000$\pm$0.000	&0.000$\pm$0.000	&0.000$\pm$0.000 	&0.000$\pm$0.000	&0.000$\pm$0.000	&0.000$\pm$0.000\\		
			\textbf{TORR}		& 3.970$\pm$0.007	& 4.097$\pm$0.199	& 5.377$\pm$2.027 	& 7.025$\pm$3.379	& 7.839$\pm$3.435	& 5.662$\pm$1.809\\	
			\textbf{TORR25}		& 3.978$\pm$0.012	& 4.207$\pm$0.324	& 5.885$\pm$2.615 	& 7.462$\pm$3.569	& 8.369$\pm$3.675	& 5.980$\pm$2.039\\		
			\textbf{RLHH}		& 3.965$\pm$0.000	& 4.244$\pm$0.544	& 5.977$\pm$2.543 	& 7.525$\pm$3.491	& 8.463$\pm$3.646	& 6.034$\pm$2.045\\		
			\textbf{DRLR}		& \textbf{3.963$\pm$0.000}	& 4.026$\pm$0.089	& 5.884$\pm$2.692 	& 7.350$\pm$3.469	& 8.325$\pm$3.669	& 5.908$\pm$1.984\\
			\textbf{SPL}		& 3.979$\pm$0.006	& 4.141$\pm$0.199	& 5.185$\pm$1.578 	& 6.413$\pm$2.562	& 7.283$\pm$2.892	& 5.400$\pm$1.447\\
			\textbf{DSPL}		& 3.972$\pm$0.007	& \textbf{4.020$\pm$0.085}	& \textbf{4.123$\pm$0.198} 	& \textbf{5.291$\pm$2.086}	& \textbf{6.444$\pm$2.997}	& \textbf{4.770$\pm$1.075}\\	
			%\bottomrule
		\end{tabularx}}
		
	\scalebox{1.00}{
		\begin{tabularx}{0.85\textwidth}{c *{5}{Y}||c}
			\toprule
			& \multicolumn{6}{c}{\textbf{Los Angeles (Corruption Ratio)}} \\
			\cmidrule(lr){2-7} 
			%\cmidrule(lr){2-6} 
			%& \underline{70\%} & 80\% & 90\% & 95\% & 70\% & 80\% & 90\% & 95\%\\
			& \textbf{10\%} & \textbf{30\%} & \textbf{50\%} & \textbf{70\%} & \textbf{90\%} & \textbf{Avg.} \\
			\midrule
			%\textbf{OLS}		&0.000$\pm$0.000	&0.000$\pm$0.000	&0.000$\pm$0.000 	&0.000$\pm$0.000	&0.000$\pm$0.000	&0.000$\pm$0.000\\		
			\textbf{TORR}		& 3.991$\pm$0.001	& 4.035$\pm$0.067	& 5.666$\pm$2.754 	& 7.569$\pm$4.098	& 8.561$\pm$4.170	& 5.964$\pm$2.218\\	
			\textbf{TORR25}		& 3.993$\pm$0.003	& 4.103$\pm$0.147	& 5.986$\pm$3.062 	& 7.834$\pm$4.181	& 8.930$\pm$4.338	& 6.169$\pm$2.346\\		
			\textbf{RLHH}		& 3.992$\pm$0.000	& 4.023$\pm$0.064	& 6.224$\pm$3.198 	& 8.013$\pm$4.179	& 9.091$\pm$4.317	& 6.268$\pm$2.352\\		
			\textbf{DRLR}		& \textbf{3.990$\pm$0.001}	& \textbf{4.016$\pm$0.031}	& 6.471$\pm$3.552 	& 8.147$\pm$4.246	& 9.197$\pm$4.341	& 6.364$\pm$2.434\\
			\textbf{SPL}		& 3.994$\pm$0.004	& 4.135$\pm$0.159	& 5.432$\pm$2.097 	& 6.856$\pm$3.109	& 7.857$\pm$3.435	& 5.655$\pm$1.761\\
			\textbf{DSPL}		& 3.992$\pm$0.021	& 4.034$\pm$0.137	& \textbf{4.510$\pm$0.599} 	& \textbf{5.717$\pm$2.237}	& \textbf{6.943$\pm$3.194}	& \textbf{5.062$\pm$1.238}\\
			\bottomrule
	\end{tabularx}}
	
\end{table*}
\subsubsection{Comparison Methods}
We used the following methods to compare the performance of the robust regression task: 
%The \textit{averaged ordinary least-squares} (\textit{OLS}) method takes the average over the regression coefficients of each mini-batch, which is computed by the ordinary least-squares method. 
\textit{Torrent} (\textit{Abbr. TORR}) \cite{bhatia2015robust}, which is a hard-thresholding based method that includes a parameter for the corruption ratio. As this parameter is hard to estimate in practice, we opted to use a variant, \textit{TORR25}, which represents a corruption parameter that is uniformly distributed across a range of $\pm 25\%$ off the true value.
We also used \textit{RLHH} \cite{rlhh17} for the comparison, which applies a recently proposed robust regression method based on heuristic hard thresholding with no additional parameters. This method computes the regression coefficients for each batch, and averages them all. 
The \textit{DRLR} \cite{Zhang2017OnlineAD} algorithm, which is a distributed robust learning method specifically designed to handle large scale data with a distributed robust consolidation.
The traditional self-paced learning algorithm (\textit{SPL}) \cite{kumar2010self} with the parameter $\lambda=1$ and the step size $\mu=1.1$ was also compared in our experiment.
For \textit{DSPL}, we used the same settings as for \textit{SPL} with the initial $\lambda_0 = 0.1$ and $\tau_\lambda=1$. All the results from each of these comparison methods were averaged over 10 runs. 

%For the classification task, we use the following comparing methods: The Support Vector Machine (\textit{SVM}) method takes the average over the \textit{RSVM}, \textit{RoLR}, Similar to the regression task, we also apply \textit{DRLR} \cite{8215535} and \textit{SPL} \cite{kumar2010self} in the classification tasks with the same loss function as SVM.

\subsection{Robust Regression in Synthetic Data} \label{section:regression_recovery}

\subsubsection{Recovery Coefficients Recovery}
Figure \ref{fig:beta} shows the coefficient recovery performance for different corruption ratios in uniform distribution. Specifically, Figures \ref{fig:beta_1} and \ref{fig:beta_2} show the results for a different number of features with a fixed data size. Looking at the results, we can conclude: 
1) Of the six methods tested, the \textit{DSPL} method outperformed all the competing methods, including \textit{TORR}, whose corruption ratio parameter uses the ground truth value. 
%2) The results of \textit{RLHH} method are significantly affected by the corruption ratio of data. When more than 40\% data is corrupted, the error of \textit{RLHH} increases dramatically. 
2) Although \textit{DRLR} turned in a competitive performance when the data corruption level was low. However, when the corruption ratio rose to over 40\%, the recovery error is increased dramatically. 
3) The \textit{TORR} method is highly dependent on the corruption ratio parameter. When the parameter is 25\% different from the ground truth, the error for \textit{TORR25} was over 50\% compared to \textit{TORR}, which uses the ground truth corruption ratio.
4) When the feature number is increased, the average error for the \textit{SPL} algorithm increased by a factor of four. However, the results obtained for the \textit{DSPL} algorithm varied consistently with the corruption ratio and feature number.
The results presented in Figures \ref{fig:beta_1} and \ref{fig:beta_3} conform that the \textit{DSPL} method consistently outperformed the other methods for larger datasets, while still achieving a close recovery of the ground truth coefficient.

\subsubsection{Performance on Different Corrupted Batches}
The regression coefficient recovery performance for different numbers of corrupted batches is shown in Table \ref{table:batchcorr}, ranging from four to nine corrupted batches out of the total of 10 batches. Each corrupted batch used in the experiment contains 90\% corrupted samples and each uncorrupted batch has 10\% corrupted samples. The results are shown for the averaged $L_2$ error in 10 different synthetic datasets with randomly ordered batches. Looking at the results shown in Table \ref{table:batchcorr}, we can conclude: 
1) When the ratio of corrupted batches exceeds 50\%, \textit{DSPL} outperforms all the competing methods with a consistent recovery error. 2) \textit{DRLR} performs the best when the mini-batch is 40\% corrupted, although its recovery error increases dramatically when the number of corrupted batch increases. 3) \textit{SPL} turns in a competitive performance for different levels of corrupted batches, but its error almost doubles when the number of corrupted batches increases from four to nine.

\subsubsection{Analysis of Parameter $\lambda$}
Figure \ref{fig:coeff_lag} show the relationship between the parameter $\lambda$ and the coefficient recovery error, along with the corresponding Lagrangian $L$. This result is based on the robust coefficient recovery task for a 90\% data corruption setting. 
Examining the blue line, as the parameter $\lambda$ increases, the recovery error continues to decrease until it reaches a critical point, after which it increases. These results indicate that the training process will keep improving the model until the parameter $\lambda$ becomes so large that some corrupted samples become incorporated into the training data. In the case shown here, the critical point is around $1.0$. 
The red line shows the value of the Lagrangian $L$ in terms of different values of the parameter $\lambda$, leading us to conclude that: 1) the Lagrangian monotonically decreases as $\lambda$ increases. 2) The Lagrangian decreases much faster once $\lambda$ reaches a critical point, following the same pattern as the recovery error shown in blue line.

\subsection{House Rental Price Prediction} \label{section:airbnb}
To evaluate the effectiveness of our proposed method in a real-world dataset, we compared its performance for rental price prediction for a number of different corruption settings, ranging from 10\% to 90\%. The additional data corruption was sampled from the uniform distribution [-0.5$y_i$, 0.5$y_i$], where $y_i$ denotes the price value of the $i\nth$ data point. Table \ref{table:rental_price} shows the mean absolute error for rental price prediction and its corresponding standard deviation over 10 runs for the New York City and Los Angeles datasets. The results indicate that: 1) The \textit{DSPL} method outperforms all the other methods for all the different corruption settings except when the corruption ratio is less than 30\%, and consistently produced with the most stable results (smallest standard deviation). 2) Although the \textit{DRLR} method performs the best when the corruption ratio is less than 30\%, the results of all the methods are very close. Moreover, as the corruption ratio rises, the error for \textit{DRLR} increases dramatically. 3) \textit{SPL} has a very competitive performance for all the corruption settings but is still around 12\% worse than the new \textit{DSPL} method proposed here, which indicates that considering the data integrally produces a better performance than can be achieved by breaking up the data into batches and treating them separately.

\section{Conclusion}\label{section:conclusion}
In this paper, a distributed self-paced learning algorithm (\textit{DSPL}) is proposed to extend the traditional \textit{SPL} algorithm to its distributed version for large scale datasets. To achieve this, we reformulated the original \textit{SPL} problem into a distributed setting and optimized the problem of treating different mini-batches in parallel based on consensus \textit{ADMM}. We also proved that our algorithm can be convergent under mild assumptions. Extensive experiments on both synthetic data and real-world rental price data demonstrated that the proposed algorithms are very effective, outperforming the other comparable methods over a range of different data settings. 

%\newpage
\renewcommand{\baselinestretch}{1}
\bibliographystyle{named}
\bibliography{ijcai18_dspl}

\end{document}